\documentclass{article}

\usepackage{microtype}
\usepackage{graphicx}
\usepackage{subfigure}
\usepackage{booktabs} %
\usepackage{xspace}
\usepackage{xcolor}

\usepackage{hyperref}

\usepackage[accepted]{icml2021}

\makeatletter
\renewcommand{\ICML@appearing}{}
\makeatother

\usepackage{amsfonts}
\usepackage{amsmath}
\usepackage{amssymb}
\usepackage{amsthm}
\usepackage{appendix}
\usepackage{mathtools}

\newtheorem{lemma}[]{Lemma}

\newtheorem{definition}[]{Definition}
\newcommand{\ie}{\textit{i.e.}\xspace}
\newcommand{\calA}{\mathcal{A}}

\newcommand{\calH}{\mathcal{H}}
\newcommand{\calI}{\mathcal{I}}
\newcommand{\calM}{\mathcal{M}}
\newcommand{\calT}{\mathcal{T}}
\newcommand{\prob}{\mathbb{P}}

\newcommand{\E}{\mathbb{E}}

\DeclarePairedDelimiterX{\infdivx}[2]{(}{)}{%
  #1\;\delimsize\|\;#2%
}
\newcommand{\rdp}{D_\alpha\infdivx}
\newcommand{\badoutcomes}{probability preservation~}

\icmltitlerunning{Defending against Reconstruction Attacks}

\begin{document}

\twocolumn[
\icmltitle{Defending against Reconstruction Attacks with Rényi Differential Privacy}

\icmlsetsymbol{equal}{*}
\newcommand{\icmlEqualContribution}{}

\begin{icmlauthorlist}
\
\icmlauthor{Pierre Stock}{ed}
\icmlauthor{Igor Shilov}{ed}
\icmlauthor{Ilya Mironov}{ed}
\icmlauthor{Alexandre Sablayrolles}{ed}
\end{icmlauthorlist}

\icmlaffiliation{ed}{Meta AI}

\icmlcorrespondingauthor{Pierre Stock}{pstock@fb.com}
\icmlcorrespondingauthor{Alexandre Sablayrolles}{asablayrolles@fb.com}

\icmlkeywords{Machine Learning, ICML}

\vskip 0.3in
]

\printAffiliationsAndNotice{} %

\begin{abstract}
Reconstruction attacks allow an adversary to regenerate data samples of the training set using access to only a trained model.
It has been recently shown that simple heuristics can reconstruct data samples from language models, making this threat scenario an important aspect of model release.
Differential privacy is a known solution to such attacks, but is often used with a relatively large privacy budget ($\varepsilon \geq 8$) which does not translate to meaningful guarantees.
In this paper we show that, for a same mechanism, we can derive privacy guarantees for reconstruction attacks that are better than the traditional ones from the literature.
In particular, we show that larger privacy budgets do not protect against membership inference, but can still protect extraction of rare secrets.
We show experimentally that our guarantees hold against various language models, including GPT-2 fine-tuned on Wikitext-103.

\end{abstract}

\section{Introduction}

Probabilistic generative models are trained to assign high likelihood to data from the training set. 
In the case of language models, the decomposition $\prob(x_1, \dots, x_T) = \Pi_{i=1}^T \prob(x_i ~|~x_{<i}) $ also allows for efficient sampling of sequences. 
Given such models will \textit{overfit} on the training set data, sampling from a trained model will sometimes yield verbatim sentences from the training set.
\citet{carlini2021extracting} leveraged this effect along with clever filtering techniques to produce samples that are likely to be in the training set.
Their work demonstrated that \emph{reconstruction attacks} are not only possible on large-scale generative language models such as GPT-2~\cite{radford2019language} but also successful: their best attack reaches $66\%$ precision on the top-$100$ generated sentences.
Another category of attacks is \emph{membership inference}, where the adversary has access to both the trained model and a data sample, and has to predict whether this data sample comes from the training set. 
This attack is easier because the task of the adversary is simpler. 

The standard defense against such privacy attacks is differential privacy (DP)~\cite{dwork2006calibrating, dwork2014algorithmic}.
DP defines a privacy budget $\varepsilon$ that can be used to control the privacy/utility tradeoff of the trained model.
However, there is no consensus over acceptable values of~$\varepsilon$ and,  in practice,~$\varepsilon$ is often chosen to defeat practical membership inference attacks~\cite{watson2021importance, carlini2021membership}.
In this paper, we show that Rényi differential privacy (RDP)~\cite{mironov2017renyi} can actually provide guarantees regarding the probability of reconstructing a sample.
The RDP analysis gives better guarantees \textit{for the same mechanism}.
In particular, we show that there is an intermediate regime in which membership inference is not protected but full reconstruction of sufficiently high-entropy secrets remains difficult.
This means that an adversary who knows a secret $s$ can decide if it was in the training set, but extracting it from the model stays hard if they do not know the secret $s$ in the first place.

We refer to samples that an adversary tries to reconstruct as \textit{secrets}.
Not all samples from the training set would be considered ``secret" in the sense that their content is public knowledge, such as newspaper headlines. 
We circumvent this issue by considering all samples secrets, and quantifying their level of secrecy by the number of unknown bits of information.
Specifically, given a probabilistic reconstruction model~$\calA$ that generates secrets from a trained model $\theta$, $s \sim \calA(\theta)$, the secrecy of a sample $s$ is $b \triangleq \log_2 \left( 1 / \prob(\calA(\theta) = s)\right)$.
An adversary needs on average $2^b$ trials to chance upon the secret $s$. 
This number of trials corresponds to the verification cost paid by an adversary in many practical scenarios.
For example if the adversary is guessing a phone number, they have to actually dial this number to ``exploit" the secret. 

\begin{figure*}
    \centering
    \includegraphics[width=0.33\textwidth]{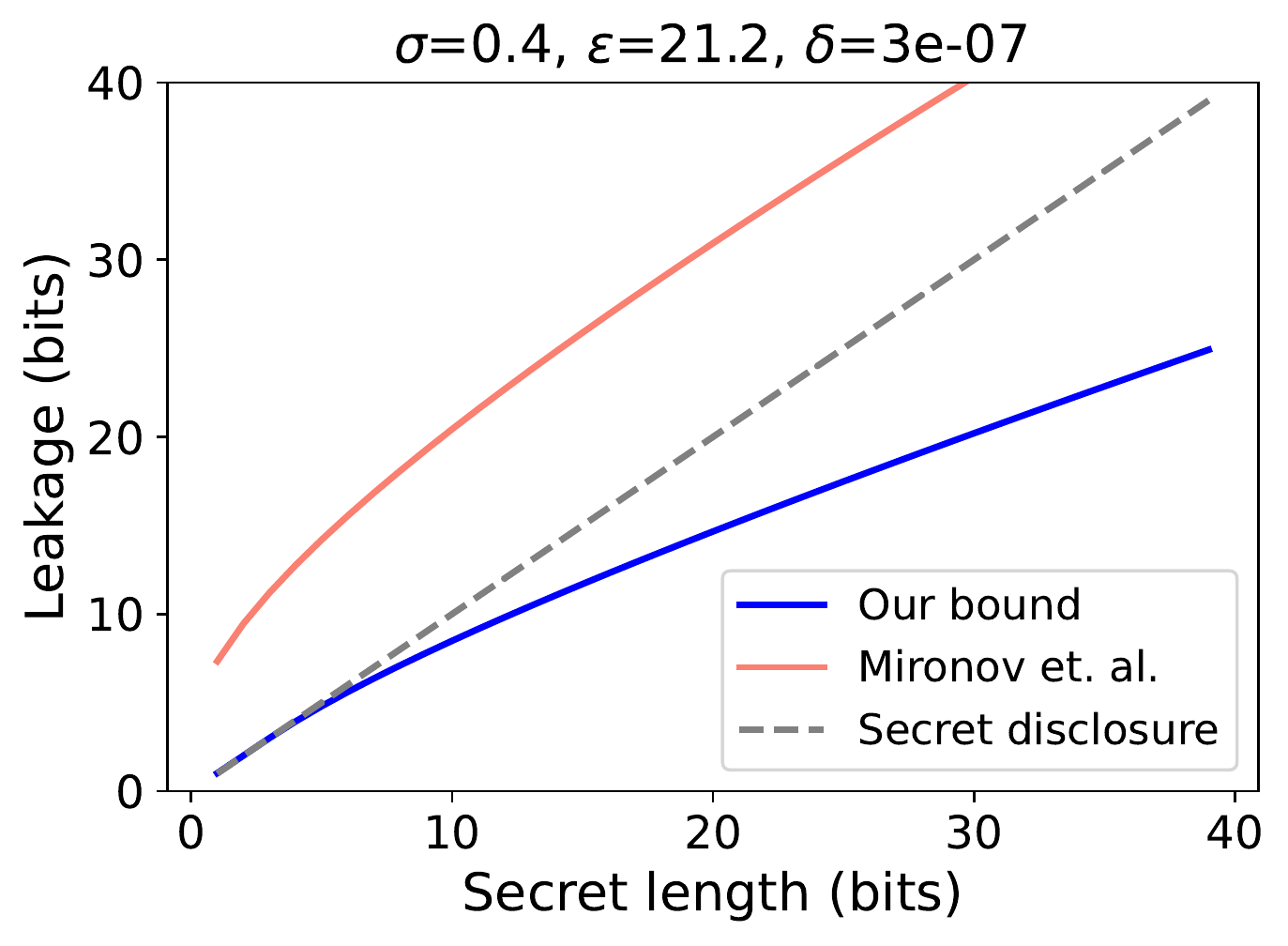}
    \includegraphics[width=0.33\textwidth]{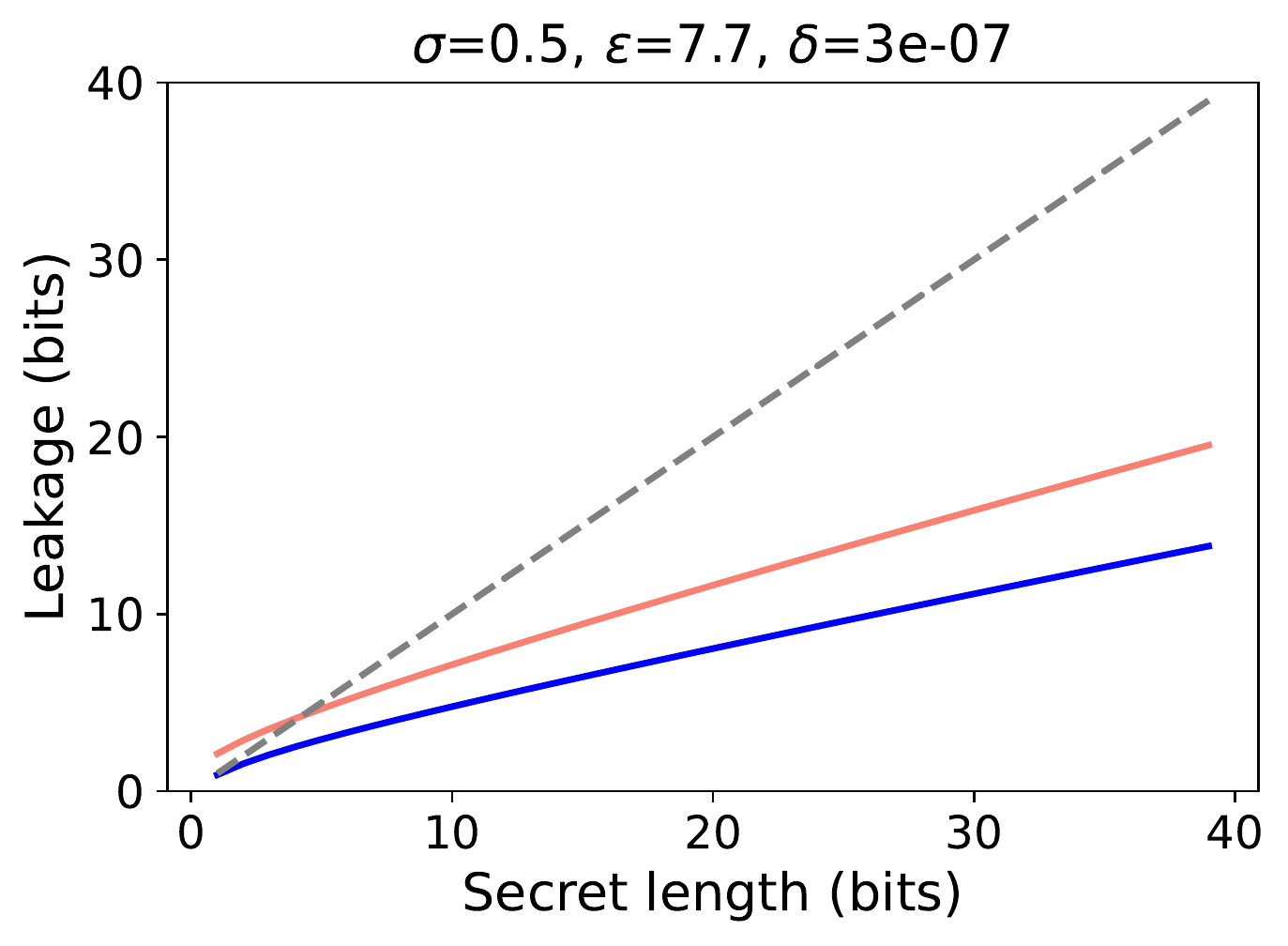}
    \includegraphics[width=0.33\textwidth]{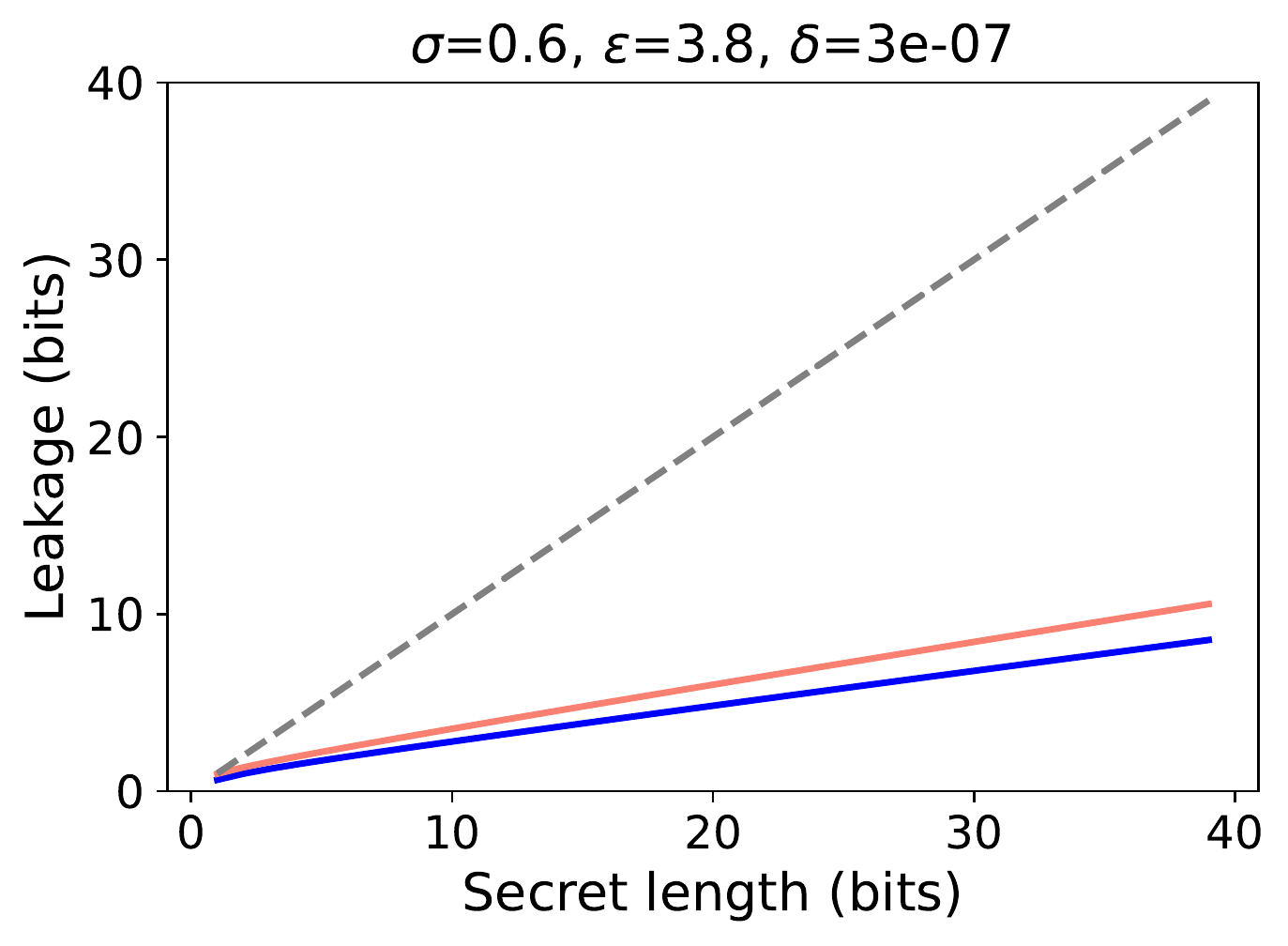}
    \vspace*{-5mm}
    \caption{\label{fig:theoretical}
    We train a Neural Network using DP-SGD and consider various levels of DP-noise $\sigma$ (increasing from left to right). We consider a secret of varying length (in bits $b$, $x$-axis) in the dataset and wish to estimate the number of bits form this secret that will leak after the training ($y$-axis). 
    For each level of DP-noise, we plot the upper bound $\min_{\alpha} h(\alpha, p_0)$ computed using \citet{mironov2019r} (see Equation~\eqref{eq:leakage_bound}) as well as our theoretical bound $L(p_0)$ against the secret length $b$ by setting $p_0=2^{-b}$.
    We additionally plot the line $y=x$: points below this line indicate that the secret will not entirely leak, points above means total leakage.
    To generate these plots, we use the settings of our real-life canary experiments in Section~\ref{subsec:wikitext}, with  $186$k training steps and a sampling rate $q=2.81\times10^{-4}$. \
    We also report the privacy budget $\varepsilon$ at $\delta=3 \times 10^{-7}$ in the plot titles computed using \citet{balle2020hypothesis}, the standard practice with DP training. The plots demonstrate that our guarantee prevents from total secret disclosure for various levels of DP noise $\sigma$.
    Furthermore, we observe that our bound is comparatively better for lower levels of the DP noise $\sigma$.
    }
\end{figure*}

A secret can have a non-zero probability even for a model that was not trained on this secret. 
As an extreme example, a language model generating uniformly random numbers will predict a $10$-digit phone number with probability $10^{-10}$.
The goal of a privacy-preserving training procedure is to ensure that the secret is not much more probable under a model that was trained on it. 

The length of the secret can vary depending on prior information: for phone numbers, knowing the area code reduces the secret from $10$ to $7$ unknown digits.
Fortunately, RDP guarantees work against \textit{any} prior knowledge, thanks to the post-processing property detailed in Section~\ref{subsec:dp}.

Our contributions are the following:
\begin{itemize}
    \item We use the \badoutcomes guarantee from RDP to derive better secret protection. In particular, we show that the length of the secret itself provides more privacy.
    \item We show empirically that our guarantee is tight for $n$-gram language models.
    \item We fine-tune language models with differential privacy and show that the observed leakage under our attack model is even smaller than the guarantee.
\end{itemize}

\section{Background}
Throughout the paper, we will use $\log$ to denote the natural logarithm and $\log_2$ to refer to the base $2$ logarithm.

\subsection{Privacy Attacks}

\paragraph{Membership Inference} attacks~\cite{homer2008resolving,shokri2017membership} determine, given a trained model and a data sample, whether the sample was part of the model's training set. 
Given a sample $x$ and model $\theta$ trained on a dataset $D$, the attacker's objective is to design a score function $\phi(\theta, x)$ such that $\phi$ is high when $x \in D$ and low otherwise.
Various score functions have been proposed, such as the gradient norm~\cite{nasr2019comprehensive}, the model confidence~\cite{salem2018ml}, or an output of a neural network~\cite{shokri2017membership}.
Surprisingly, choosing the score function to be the loss $\mathcal L(\theta, x)$ is a simple and effective approach~\cite{sablayrolles2019white,salem2018ml}.

Recent works argue that a practically relevant threat model of membership inference is to confidently predict training set membership of a few samples rather than guessing well on average~\cite{watson2021importance, carlini2021membership,ye2021enhanced}.
Such membership inference attacks are evaluated using true and false positive rates.
In this context, calibrating the loss by centering~\cite{watson2021importance} or by fitting a Gaussian likelihood~\cite{carlini2021membership} further improves the performance of the attack and yields state-of-the-art results.

This kind of attacks essentially identifies uncommon samples on which a model is overconfident. 
These samples are also the focus of our study.

\paragraph{Language Model Extraction.} 
Disclosure of \emph{verbatim} training sentences by state-of-the-art generative sequence models poses privacy issues. 
For instance, commercial auto-completion networks trained on sensitive user data may surface exact training samples such as credit card numbers with an appropriately chosen prompt. \citet{carlini2019secret} propose a methodology to quantify the exposure of unique and secret training sequences to such extraction attacks. 
To do so, the authors insert a \emph{canary} in the training set and measure its exposure as the excess belief that model has in the canary over random chance. 
The authors then conclude that differential privacy is a suitable mitigation technique, although at the cost of some utility. Similarly, \citet{carlini2021extracting} show that it is possible to extract hundreds of training sentences in a two-step procedure when attacking GPT-2, a language model trained on large scrapes of the public internet \cite{radford2019language}. 
First, the authors generate 200,000 samples for each of the following sampling strategies: sampling with a linearly decaying temperature, top-$k$ sampling and sampling conditionally on random Internet text. 
Then, they reduce the problem to Membership Inference among the generated samples and select the top-100 samples most likely to be members according to different filtering strategies.
They then manually check for each sentence that it was indeed in the training set by querying a search engine indexing the training set. 
Depending on the metrics to rank the generated sentences (loss of the target model calibrated or not with a loss of a smaller model for instance), the authors identify a few dozens of training sentences from the top-100 sentences.

\subsection{Differential Privacy}
\label{subsec:dp}
Differential Privacy is a leading standard for privacy guarantees~\cite{dwork2006calibrating,dwork2014algorithmic}.
\begin{definition}
    A randomized mechanism $\calM \colon \mathcal D \to \mathcal R$ satisfies $(\varepsilon, \delta)$-differential privacy (DP) if, for any adjacent inputs $D, D' \in \mathcal D$ and for any $S \subset \mathcal R$, we have
    \begin{equation*}
        \prob[\calM(D) \in S] \leq e^\varepsilon \prob[\calM(D') \in S)] + \delta. 
    \end{equation*}
\end{definition}
Datasets $D$ and $D'$ are \emph{adjacent} if they differ by at most one element. 
Originally, DP was defined with $\delta=0$, a setting called \textit{pure} DP.
Pure DP can be interpreted as a bound on the change in information contained in any release $\theta$, defined as $\calI(D, \theta)\triangleq-\log_2 \left( \prob(\calM(D) = \theta) \right)$, and guarantees that this quantity does not vary by more than~$\varepsilon$ between adjacent datasets. 
Unfortunately, mechanisms using Gaussian noise usually do not satisfy pure DP, thus the additional term $\delta > 0$ was introduced.

Furthermore, analysis of the subsampled Gaussian mechanism~\cite{abadi16deep,mironov2019r} requires the introduction of Rényi differential privacy (RDP)~\cite{mironov2017renyi} to leverage tighter composition properties. 
We recall here the properties of RDP that will be useful for the rest of the paper, while referring the reader to \citet{mironov2017renyi} for a more comprehensive overview.

\begin{definition}
    For two probability distributions $P$ and $Q$  defined over $\mathcal R$, the Rényi divergence of order $\alpha > 1$ is 
    \begin{equation*}
        \rdp{P}{Q} \triangleq \frac{1}{\alpha - 1}\log \E_{x\sim Q}\left(\frac{P(x)}{Q(x)}\right)^\alpha.
    \end{equation*}
\end{definition}

\begin{definition}
    A randomized mechanism $\calM \colon \mathcal D \to \mathcal R$ satisfies $(\alpha, d_\alpha)$-Rényi differential privacy (RDP) if, for any adjacent inputs $D, D' \in \mathcal D$, we have
    \begin{equation*}
        \rdp{\calM(D)}{\calM(D')} \leq d_\alpha.
    \end{equation*}
\end{definition}
RDP guarantees can be translated back into DP guarantees while the converse is not true, making RDP a strictly stronger property. 
If $\calM$ is $(\alpha, d_\alpha)$-RDP, then it is also $\left(d_\alpha + \frac{\log 1 / \delta}{\alpha - 1}, \delta\right)$-DP for any $0 < \delta < 1$. The bound is slightly improved by~\citet{balle2020hypothesis}.

\paragraph{Post-processing.}
As with $(\varepsilon, \delta)$-DP, $(\alpha, d_\alpha)$-RDP guarantees are preserved by post-processing: if $\calM$ is $(\alpha, d_\alpha)$-RDP and if $\mathcal A\colon \mathcal R \to \mathcal R'$ is a randomized mechanism, then $\mathcal{} A \circ \calM$ is $(\alpha, d_\alpha)$-RDP. 

\paragraph{Probability Preservation.}
A direct consequence of $(\alpha, d_\alpha)$-RDP is quantified by the following inequality:
\begin{equation}\label{eq:logratio}
    e^{-d_\alpha}p^{\alpha/(\alpha-1)}\leq p' \leq \left(e^{d_\alpha} p\right)^{(\alpha-1)/\alpha}
\end{equation}
where $p \triangleq \prob[\calM(D) \in S]$ and $p'\triangleq\prob[\calM(D') \in S]$. Informally, since $\calM(D)$ and $\calM(D')$ are close, the probabilities of any event $S$ under both $\calM(D)$ and $\calM(D')$ are also close.

In the remainder of the paper, we adapt the definitions above to the setting of Deep Learning. 
The dataset $D \in \mathcal D$ contains training samples $\{x^{(1)}, \dots, x^{(n)}\}$ and $\calM$ denotes a randomized training procedure. 

\paragraph{DP-SGD.} 
Private training of neural networks is usually conducted with DP-SGD~\cite{abadi16deep,song2013stochastic,bassily2014private} as follows. 
For each training step~$t$, we gather a batch (of variable length~$L$) of training examples $x^{(i)}$ by sampling each element without replacement with probability (or sampling rate) $q$. 
Then, we compute the per-sample gradients $g_t(x_i)$. Next, we bound the contribution of each sample by clipping the per-sample gradients using a clipping constant $C > 0$:
\begin{equation*}
    \bar g_t(x_i) = {g_t(x_i)} / \max(1, \|g_t(x_i)\|_2 / C).
\end{equation*} 
Finally, we compute the gradient for the batch $\tilde g_t$ as:
\begin{equation*}
    \tilde g_t = \frac{1}{L}\left(\sum_i \bar g_t(x_i) + \mathcal N\left(0, \sigma^2C^2\right)\right).
\end{equation*}
and use $\tilde g_t$ in the optimization step. 
Finally, an \emph{accountant} tracks the privacy budget $\varepsilon$ over all the training steps. 
This budget depends on $\delta$, on the sampling rate $q$, on the number of training steps and on the noise multiplier $\sigma$. 

\paragraph{NLP with Differential Privacy.}
Recent work has shown that fine-tuning language models to competitive accuracy with differential privacy is possible. 
For instance, \citet{li2021large} provide a recipe to directly fine-tune large transformer models with $50$ to $300$ million parameters directly with DP at privacy level $\varepsilon \in \{3,  8\}$ for various downstream tasks. Similarly, \citet{yu2021differentially} finetune privately competitive RoBERTa-Large models \cite{liu2019roberta} with $\varepsilon = 6.7$ by training only a fraction of the network's parameters using low-rank weight matrices or skip-connections.

\section{Our method}

We first derive an upper bound on the information leakage that is better than the one from \citet{mironov2019r}. Then, we present a \emph{lazy} method to identify samples that are likely to leak when performing a reconstruction attack.

\subsection{Reconstruction and Canaries}

The goal of reconstruction attacks is to surface training samples given access to the target network's weights. 
If $D$ is a dataset and $s$ a secret, we denote by $D' = D \cup \{s\}$. 
Recall that $\calM(D)$ (resp., $\calM(D')$) represents the distribution of the target network's weights after training on $D$ (resp., $D'$). 
We assume $\calM$ is $(\alpha, d_\alpha)$-RDP. 
Finally, $\calA$ denotes the attack mechanism that takes the target network's weights and outputs a probability distribution over secrets. 
For a given secret s, we note $p_0 \triangleq \prob[\calA(\calM(D)) = s]$ and $p_1 \triangleq \prob[\calA(\calM(D')) = s]$.
We assume that the dataset $D$ does not contain information about the secret $s$; we will thus sometimes refer to $p_0$ as the prior, and $p_1$ as the posterior.
Then, using the Probability Preservation inequality~\eqref{eq:logratio},

\begin{equation}
    \label{eq:leakage_bound}
    \underbrace{-d_\alpha - \frac{\log(1/p_0)}{\alpha - 1}}_{\triangleq -h(\alpha, p_0)} \leq \log\left(\frac{p_1}{p_0}\right)  \leq \underbrace{d_\alpha \frac{\alpha -1}{\alpha} + \frac{\log (1/p_0)}{\alpha}}_{\triangleq l(\alpha, p_0)},
\end{equation}
Since the \emph{leakage} $\log\left(\frac{p_1}{p_0}\right)$ is independent of $\alpha$, we may strengthen the bound by taking a minimum over all orders:
\begin{equation}
    \label{eq:min_bound}
    L(p_0)\triangleq \min_{\alpha>1}~l(\alpha, p_0).
\end{equation}
Finally, since $\alpha > \alpha - 1$, we have $l(\alpha, p_0) < h(\alpha, p_0)$.

\paragraph{Comparison with traditional DP guarantees.}
Equation~\eqref{eq:leakage_bound} implies a bound on the absolute leakage 
\[
    \left| \log \left(\frac{p_1}{p_0} \right) \right| \leq \max(l(\alpha, p_0), h(\alpha, p_0)) = h(\alpha, p_0).
\]
If we take $\delta = p_0$, this corresponds to the traditional $\varepsilon$ given by $ \min_\alpha h(\alpha, \delta)$.
Given that $l(\alpha, p_0) < h(\alpha, p_0)$, our bound on leakage is \textit{better} than traditional guarantees from $(\varepsilon, \delta)$-DP.
These values are shown in Figures~\ref{fig:theoretical} and~\ref{fig:leakage_steps}: for lower values of the noise multiplier $\sigma$, the bound provided by $l$ is much lower than $h$, while the gap between the two decreases as the noise multiplier $\sigma$ gets bigger.

We can extend that analogy and compare our leakage guarantee to the privacy budget $\varepsilon$ that would be given for a probability of failure $\delta = p_0$. 
Numerically, our bounds are also better than the slightly tighter bounds provided by \citet{balle2020hypothesis}.
Nominally, a bound on the absolute leakage is a stronger guarantee, because it also prevents the posterior $p_1$ from becoming \emph{smaller} than the prior $p_0$.
However, we argue that this case is much less relevant from the risk perspective: even if the probability of other outcomes $s'$ becomes less likely, ``mechanically" increasing the probability of the secret $s$, the probability of the secret is still bounded by the right-hand side of Equation~\eqref{eq:leakage_bound}.
Furthermore, negative membership inference, \ie, predicting that a sample was \emph{not} part of a training set, is not as powerful as positive membership inference.
An individual can be present in a data collection pool, but not included in a particular training set for a variety of reasons, hence absence from the training set does not imply absence from the data collection stage.

Finally, thanks to post-processing, Equation~\eqref{eq:leakage_bound} is true \emph{regardless of the attack mechanism} that is considered. 
Indeed, if we denote by $\calM$ the training mechanism and by $\calA$ a reconstruction attack, attacks on trained models are given by $\calA \circ \calM$ and thus enjoy RDP guarantees.

\begin{figure}
    \centering
    \includegraphics[width=0.5\textwidth]{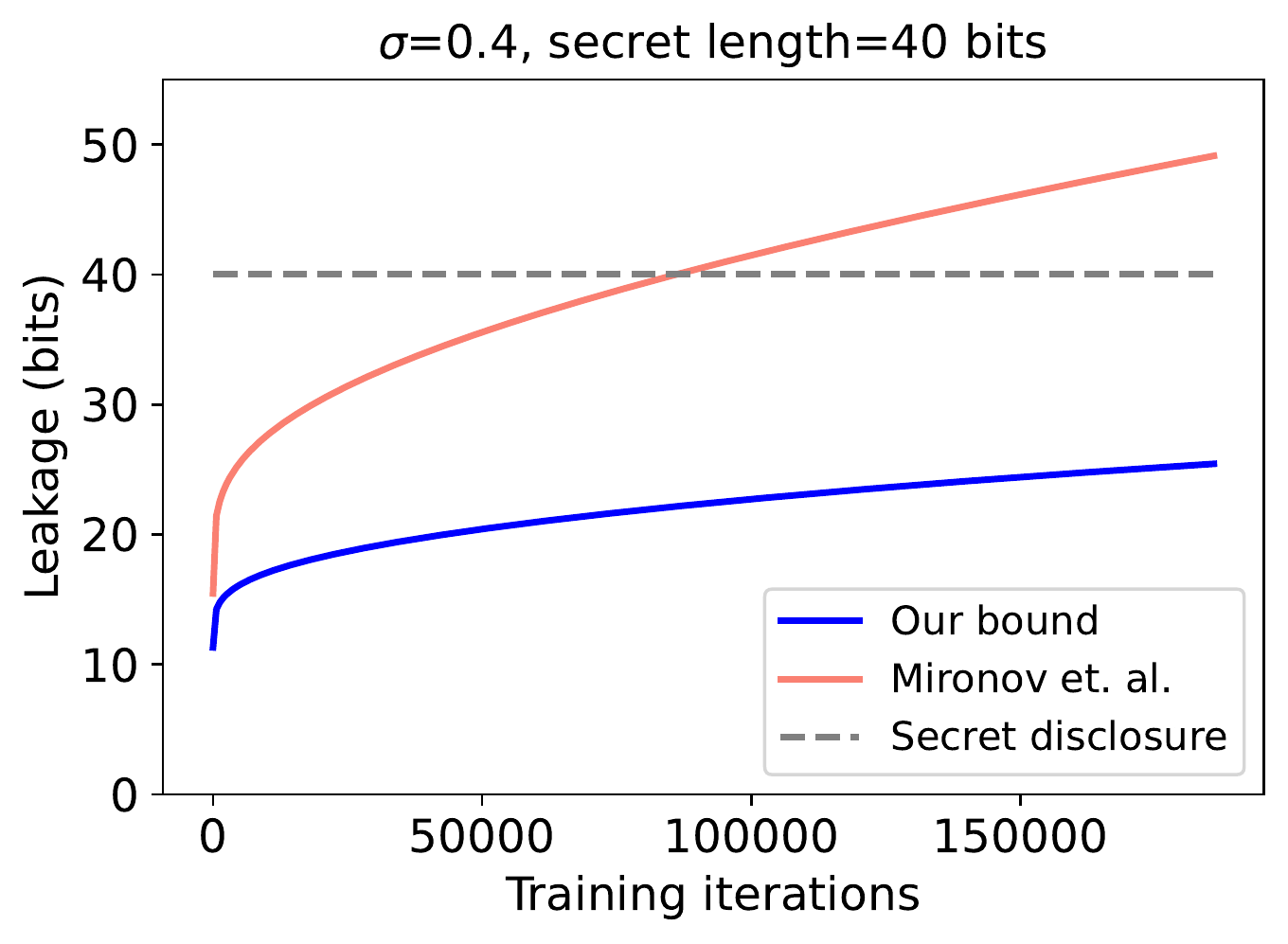}
    \vspace*{-5mm}
    \caption{\label{fig:leakage_steps}
    We consider training a Neural Network with DP-SGD with the setup described in Figure~\ref{fig:theoretical}, corresponding to our real-life experiments detailed in Section~\ref{subsec:wikitext}. Here, we fix the secret length to $b=40$ bits and plot the dependency of the privacy budget with the number of training iterations with an amount of DP noise $\sigma = 0.4$.
    Interpretation: our bound $L(p_0)$ shows that the leakage is only partial even for a large number of training iterations, contrary the bound $\min_{\alpha} h(\alpha, p_0)$ \cite{mironov2019r} that becomes vacuous after a certain amount of iterations.
    }
\end{figure}

\paragraph{Length of the Secret.}
With a slight abuse of notation, let us consider the leakage (in bits) as a function of the number of secret bits $L_2(b) \triangleq \min_\alpha d_\alpha \frac{\alpha -1}{\alpha \log(2)} + \frac{b }{\alpha}$.
Our leakage function $b \mapsto L_2(b)$ satisfies two properties: it is non-decreasing and concave (see Appendix~\ref{app:proof_leakage} for proof).
Thus, longer secrets will leak more in absolute terms, but less in relative terms.
Let us assume that we are looking at a particular secret of binary nature (such as whether some individual owns a particular object for example), with prior $p_0=2^{-b}$.
Even though there are two possible outcomes, the prior $p_0$ is not necessarily equal to $1/2$: if the item is rare, the prior is more likely to be smaller.
The log-posterior $\log_2(p_1) < -b+L_2(b)$, a non-increasing function (see Appendix~\ref{app:proof_leakage} for proof).
This upper-bound is minimized when $b=0$: longer secrets lead to smaller values of $p_1$.
The most sensitive secrets are the ones that are the most likely in the first place, and the length of the secret itself acts as a protection against reconstruction attacks.

In particular, membership inference attacks correspond to attacks with a low number of secret bits.
These attacks are usually conducted with a prior probability of $1/2$~\cite{yeom2018privacy,sablayrolles2019white}, and even though some works consider different ratios of positive to negative members~\cite{watson2021importance,rezaei2021difficulty}, they stay within a factor of $1$ to $10$.
Some privacy settings (noise level $\sigma$, sampling rate $q$ and number of steps) will thus not protect against membership inference (because of the low number of bits to guess) but still will not allow for full reconstruction of the rarest samples (high number of bits).

\subsection{Lazy Sampling: Identifying Samples Likely to Leak}

Let us assume that we want to generate a secret $s$ from a model $\theta$, using the probabilistic attack model $\calA(\theta)$. 
The method of \citet{carlini2021extracting} requires sampling hundreds of thousands of times from $\calA(\theta)$ in the hope that we find the secret $s$ (and then filter out most of the generated samples). 
Fortunately, most attacks of \citet{carlini2021extracting} are amenable to \textit{lazy sampling}: 
we can take a secret $s$, directly compute $p = \prob(\calA(\theta) = s) $ and estimate that we need in order of $1 / p$ samples from $\calA$ to generate $s$.

Not all probabilistic processes $\calA$ have a tractable density $\prob(\calA(\theta) = s)$.
For instance, generating a sequence $x_1, \dots, x_T$ and only keeping the last samples $x_i, \dots, x_T$ does not have a tractable density.
On the other hand, vanilla sampling from a language model $\theta$ can be done lazily because the probability of a sequence $x_1, \dots, x_T$ can be expressed as $ \Pi_{i=1}^T~f_{\theta, i}(x_i~|~x_{1:i}) $ where $x_{1:i} = (x_1, \dots, x_{i-1})$ and $f_{\theta, i}(x_i~|~x_{1:i})$ is the conditional probability of $x_i$ given the past.
While this is straightforward for regular sampling, it is also true for temperature and top-$k$ sampling, with the caveat that these probabilities can sometimes be $0$, as illustrated in Figure~\ref{fig:topk_proba}. 

We define $T_k(\theta, x_{1:i})$ the set of top-$k$ predictions and $ \beta_1, \dots, \beta_T$ a set of temperatures.
From there, we can define a top-$k$ and/or temperature language model as 

\begin{equation*}
    \lambda(x_i~|~x_{1:i}) \triangleq \begin{cases}
          \frac{f_\theta(x_i\mid x_{1:i})^{\frac{1}{\beta_i}}}{\sum_{y \in T_k(\theta\mid x_{1:i})} f_\theta(y\mid x_{1:i})^{\frac{1}{\beta_i}}} \, &\text{if} \, x_i \in T_k(x_{1:i}) \\
          0 \, &\text{otherwise.} \\
     \end{cases}
\end{equation*}
We can then define the top-$k$ probability of the whole sequence by applying the chain rule:
\begin{equation*}
    \lambda(x_1, \dots, x_T) = \Pi_{i=1}^T~\lambda(x_i~|~x_{1:i}).
\end{equation*}

\begin{figure}
    \centering
    \includegraphics[width=0.49\textwidth]{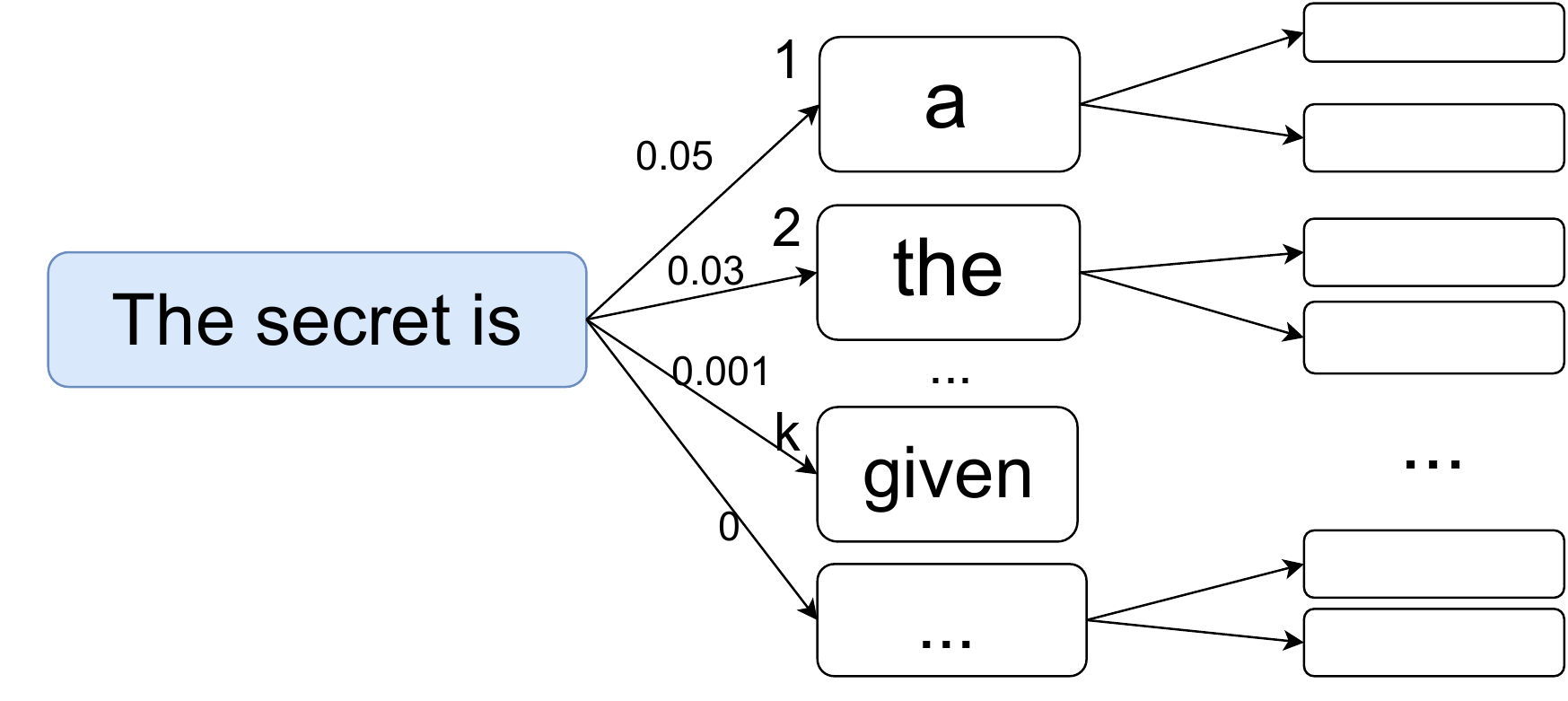}
    \caption{\label{fig:topk_proba}
    Top-$k$ sampling from a language model has a tractable density.
    The probability of a sequence is the product of the conditional probabilities along the path. 
    If one word does not belong to the top-$k$ predictions, its probability will be $0$, thus making the probability of the entire sentence $0$.
    }
    
\end{figure}

Lazy sampling allows us to analyze two of the main strategies of \citet{carlini2021extracting}, but does not apply to the Internet sampling one.
Indeed, Internet sampling first chooses a text $c$ crawled from the public web, and generates iteratively from it $f_\theta(\cdot, c)$, effectively using  $c$ as a prompt.

\paragraph{Experimental Protocol.}
In practice we observe that the secrets revealed do not depend much on the model $\theta$, given a fixed dataset $D$. 
In the remainder of this paper (except for Section~\ref{subsec:tabular}), we will make the assumption that for any $\theta \sim \calM(D)$ and $\theta' \sim \calM(D)$, 
\begin{equation*}
    \log \prob(\calA(\theta) = s) \approx \log \prob(\calA(\theta') = s).
\end{equation*}
To support our assumption, we measured the mean and standard deviation of the log-probability of the secret $\log(p_1)$ for our $n$-gram model detailed in Section~\ref{subsec:tabular}. 
We consider $10{,}000$ models for various levels of DP noise $\sigma$ and observe the probability of leakage is tightly concentrated around the mean. For instance, with $\sigma=2.875$, $\log p_1$ has mean $\simeq -22.5$ and standard deviation $\simeq 1.5$, from which we conclude that the log-probability of a secret does not vary much when re-training a model on a fixed dataset~$D$.

Using this assumption, we can take all sentences $s=x_1, \dots, x_T$ from the training set, compute their probability of being generated by the sampling stage, and compare it to their ``score" that is used at the filtering stage. 

The above strategy allows us to identify samples that would be reconstructed using this particular attack.
Of course, if this particular attack fails to reconstruct some sample $s$, it does not mean that all attacks will fail to reconstruct $s$.

\section{Experiments}

We first experiment with a simple $n$-gram language model to empirically compare the secret leakage with our bounds. %
We then consider a non-private language model that exhibits samples that are the most at risk for a reconstruction attack using the lazy strategy for the \citet{carlini2021extracting} attack. 
Finally, we experiment on private fine-tuning of GPT-2, with parameters yielding low perplexity~\cite{li2021large}, and show that the empirical leakage we measure is even smaller than predicted by our bound.

\subsection{$n$-gram Language Model}
\label{subsec:tabular}

We first experiment on a $n$-gram language model to compare our guarantees against empirical secret leakage. 
We consider a number of time steps $T > 0$ and generate a random sequence of digits $c \in \{0, 1, \dots, D-1=9\}^T$ (the \emph{canary}). 
Our objective is to train the $n$-gram language model on the (fixed) canary $c$ only.

A full $n$-gram language model $f_\theta$ has $\sum_{i=1}^T 10^i \approx 10^T$ parameters.
However, the only parameters of interest in our case are the ones corresponding to $f_\theta(c_i | c_{1:i})$.
These parameters suffice to compute the probability of the sequence $c$ according to the model $f_\theta$, and the rest of the parameters correspond to other ``branches" of the language model.
Hence, we can train $f_\theta$ lazily by only modifying these parameters, which brings the number of parameters down to $10 T$.

We train our model with the softmax loss. %
More specifically, our loss $\mathcal L$ writes
\begin{equation*}
    \mathcal L(\theta) = -\sum_{t=0}^{T-1} \log(u_{t, c_t}),
\end{equation*}
where $c_t$ is the digit at index $t$ in the canary and $u_{t,i}$ is the probability that digit $i$ appears at position $t$ in the canary:
\begin{equation*}
    u_{t, i} = \frac{e^{\theta_{t, i}}}{\sum_{d=0}^{D-1} e^{\theta_{t, d}}}.
\end{equation*}

We experiment with Opacus~\cite{yousefpour2021opacus}, using a randomly generated and fixed canary of length $T=10$.
We set the sampling rate to $q=1$, the clipping factor $C=1$, the DP noise level $\sigma \in [0.5, 10]$ and learning rate $\eta = 0.5 / \sigma$. 
We now wish to measure empirically the leakage $\log\left(p_1/p_0\right)$. 
We have:
\begin{align*}
    p_1 &= \prob(\calA(\calM(D')) = s) \\
    &= \E_{\theta \sim \calM(D')}(f_\theta(c))
\end{align*}
We approximate $p_1$ using Monte-Carlo sampling, by training $N=10{,}000$ models and computing
\begin{align}
\label{eq:empirical_leakage}
    \log(p_1) \approx \log \left( \frac{1}{N} \sum_{i=1}^N f_{\theta_i}(c) \right).
\end{align}
Given that the probabilities $f_{\theta_i}(c)$ can be quite small, we perform computations in the log-space for numerical stability. 
We have $p_0 = D^{-T}$ since the problem is invariant to the permutation of the canary digits. 

Finally, we use the \texttt{RDPAccountant} from Opacus to compute $d_\alpha$ for a range of orders $\alpha \in [1.01, 63]$ and compute $L(p_0)$ as in Equation~\eqref{eq:min_bound}. The results are shown in Figure~\ref{fig:sigma_leak}, where the empirical leakage refers to $\log\left(p_1/p_0\right)$ and where our theoretical bound refers to $L(p_0)$. We observe that the bound is tight for this simple $n$-gram language model.
 
\begin{figure}
    \centering
      \includegraphics[width=0.5\textwidth]{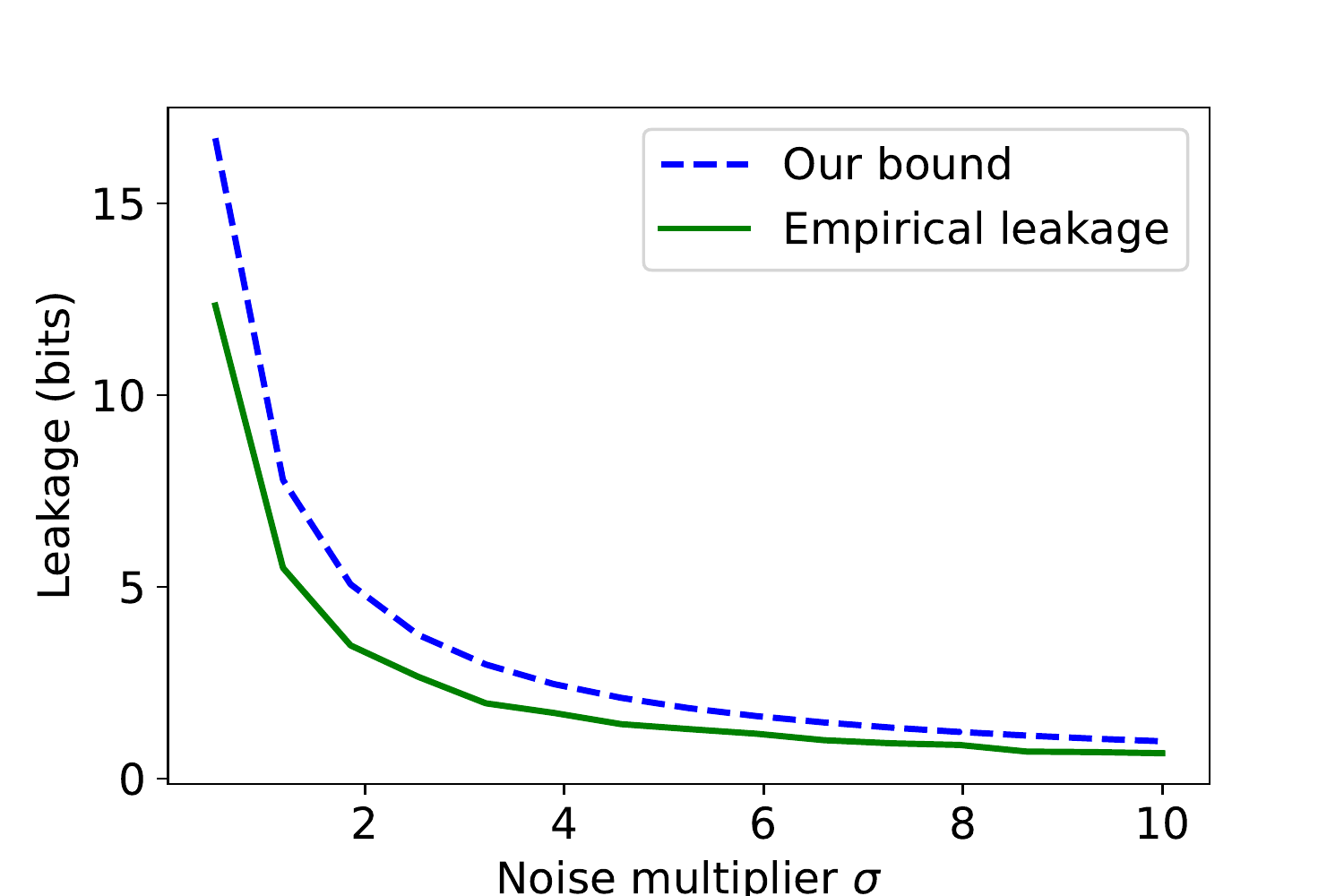}
    \vspace*{-5mm}
    \caption{
    \label{fig:sigma_leak}
    Empirical leakage $\log\left(p_1/p_0\right)$ computed using Equation.~\eqref{eq:empirical_leakage} on $n$-gram language models, compared to our theoretical bound  $L(p_0)$ defined in Equation~\eqref{eq:min_bound}. 
    We can see that our bound is very tight.
    }
\end{figure}

\subsection{Lazy Reconstruction without Defenses}

\begin{figure*}
    \centering
    \includegraphics[width=0.48\textwidth]{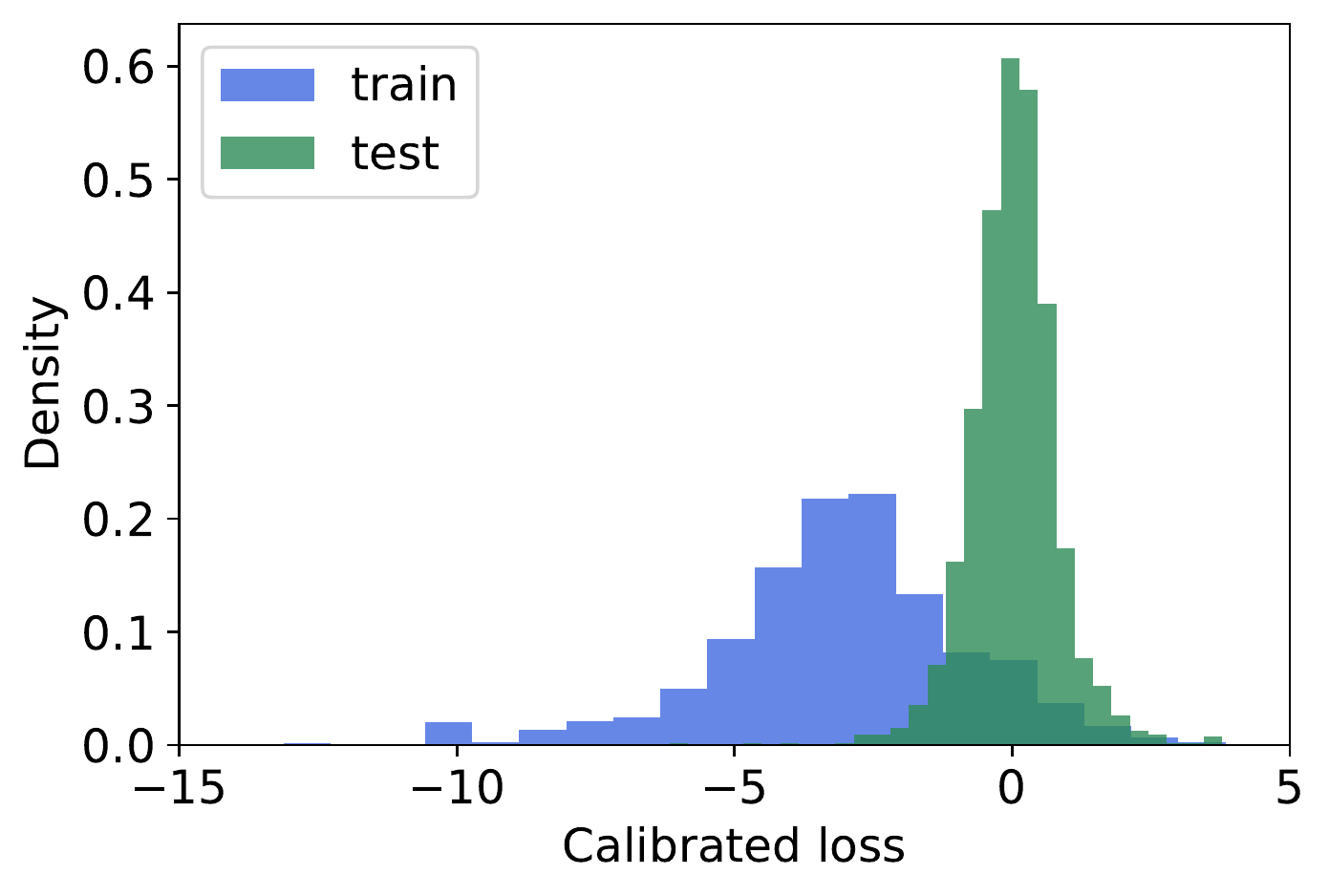}
    \includegraphics[width=0.505\textwidth]{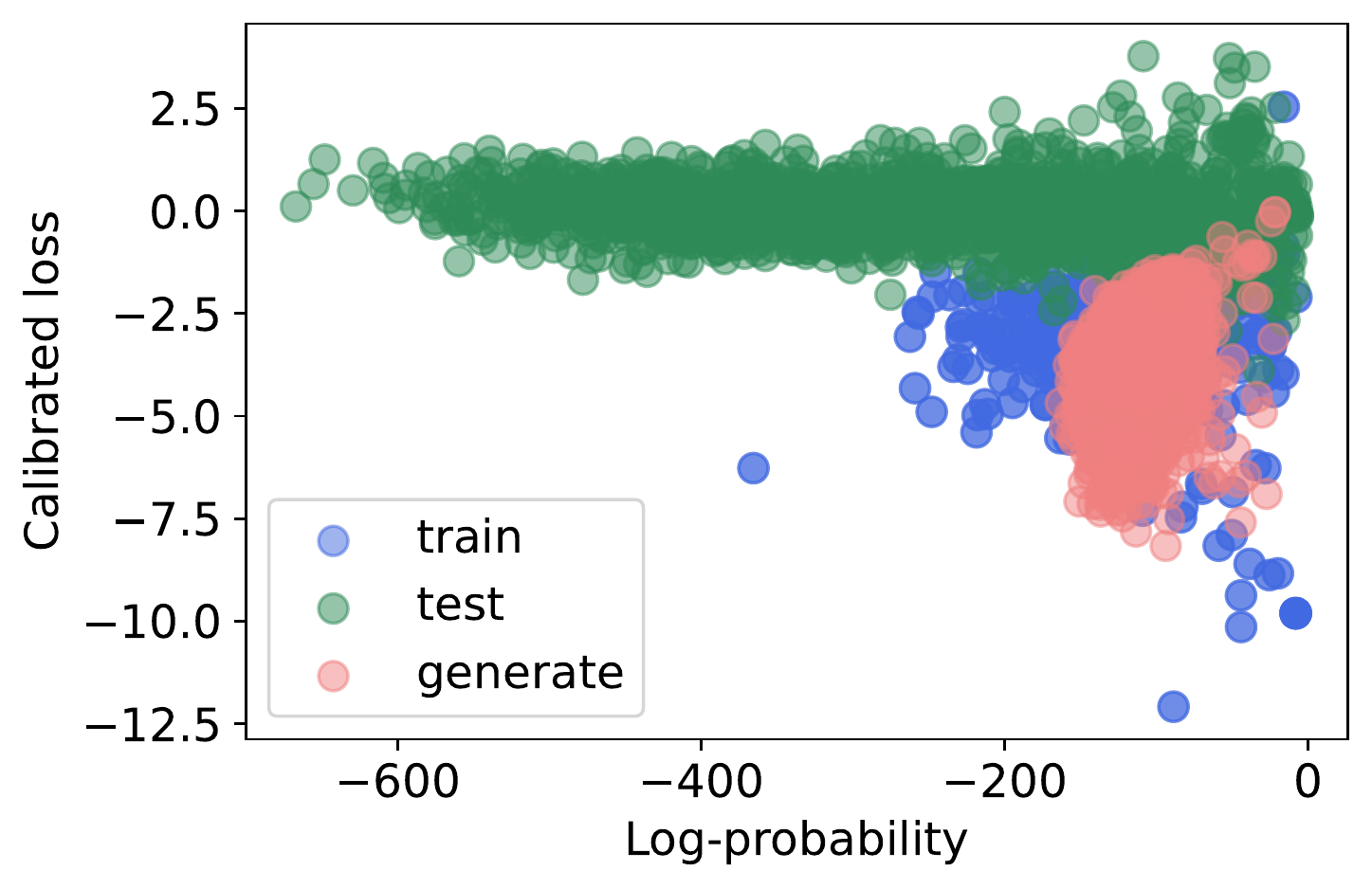}
    \vspace*{-5mm}
    \caption{\label{fig:proba_diff_loss}
    We train a vanilla language model without DP and compute, for a given sentence, its calibrated loss using a model that is not trained on the same data, as well as its log-probability predicted by the target model. Left: we plot the histogram of a random subset of the train and test sets 
    Right: for a given sentence, we plot its calibrated loss ($y$-axis) against its log-probability of being generated by the target model ($x$-axis). 
    In addition to sentences from the train and test set, we generate sentences from the target model and note that their average calibrated loss is lower than for sentences of the training set. 
    The samples from the train set that are the most at-risk for a reconstruction attack are these with high log-probability (high probability of being generated by the attack model) and low calibrated loss (high probability of having been memorized by the target model). See Table~\ref{tab:samples} for selected sentences from the training set.}
\end{figure*}

We now consider a vanilla target model trained without DP on the OpenWebText dataset~\citep{OpenWeb} and want to identify samples from the training set that are most at risk in the event of a reconstruction attack. We compute, for a given sentence~$x$, its calibrated loss defined as
\begin{equation*}
    \mathcal L(\theta, x) - \mathcal L(\theta_0, x),
\end{equation*}
where $\theta$ represents the target model's weights and $\theta_0$ represents a calibration model that is trained on some calibration data. In practice, we consider 3 subsets of OpenWebText as our training, calibration and test set of 1 million samples each. %
We train our models using the setup of \citet{radford2019language}. 
As shown in Figure~\ref{fig:proba_diff_loss} on the left plot, we observe that samples from the training set have a lower calibrated loss, indicating that they have been memorized by the model. 
As expected, test samples have an average calibrated loss around $0$ with a small standard deviation.
On the right plot, we observe that the average calibrated loss for sentences generated by the target model is lower than for sentences from the training set. We consider that the samples from the train set the most at-risk for a reconstruction attack are these with high log-probability (high probability of being generated by the target model) and low calibrated loss (high probability of being memorized by the target model) 

Table~\ref{tab:samples} shows typical examples of elements from the training set with various levels of risk.
Some of the sentences likely to be generated are common (such as the two first rows).
Other sentences however are very likely to be sampled (with a probability close to $10^{-3}$) and have a low calibration score, indicating they were memorized by the model.
Finally, some sentences have a very low probability of being generated in the first place (less than 1 in $10$ billions), so even though they are memorized by the model (low calibration score), they are unlikely to be surfaced.
Note that \citet{carlini2021extracting} showed an analogous plot by displaying perplexity of the trained model \textit{versus} zlib entropy.
There are two major differences: our $x$-axis shows the probability of being generated by the attack model (which is different from perplexity due to temperature and/or top-$k$ sampling), and our analysis is conducted on the \emph{training} set, whereas their analysis is done on generated samples.
In particular, we can have confidence that certain sentences with a very low probability of being generated will not be recovered 
\textit{by this particular strategy}.

\subsection{Real-life Canary Experiments}
\label{subsec:wikitext}

We also experiment with a realistic scenario of privately fine-tuning a large language model. 
We take a pre-trained GPT-2 model~\cite{wolf-etal-2020-transformers} and fine-tune it on Wikitext-103~\citep{Wikitext-103}. 
We select the hyper-parameters to achieve the best perplexity at a given privacy budget.
In particular, we follow the findings of \citet{li2021large} and use a large batch size ($B=1024$), a small clipping threshold ($C=1$), use AdamW~\cite{loshchilov2018decoupled} and freeze the embedding layers. 
We also find that using a low learning rate ($\mathit{lr}=0.0001$) is crucial to avoid divergence.
With these hyperparameters, we are able to fine-tune GPT-2 on Wikitext-103 and reach a perplexity of $45$, which is very close to the non-private fine-tuning perplexity of $38$ that we obtain. 
These results echo the findings of \citet{li2021large}.

We add a canary sentence to the training dataset. 
The canary consists of a prompt (``John Doe's credit card number is") and a secret, a string of 16 random digits. 
Given the trained language model, we approximate $p_1$ by computing the probability of the secret given the prompt $p_1 \approx f_\theta(\textrm{number}~|~\textrm{``John Doe's credit card number is"})$. 
With the chosen privacy parameters ($\sigma \in [0.3, 0.5]$, $q=2.81 \times 10^{-4}$ and $186$k steps), we observe that the empirical leakage is negligible. 
To make our method more sensitive, we upsample the canary up to $280$ times compare to regular samples.

\begin{figure}[b!]
\centering
      \includegraphics[width=0.5\textwidth]{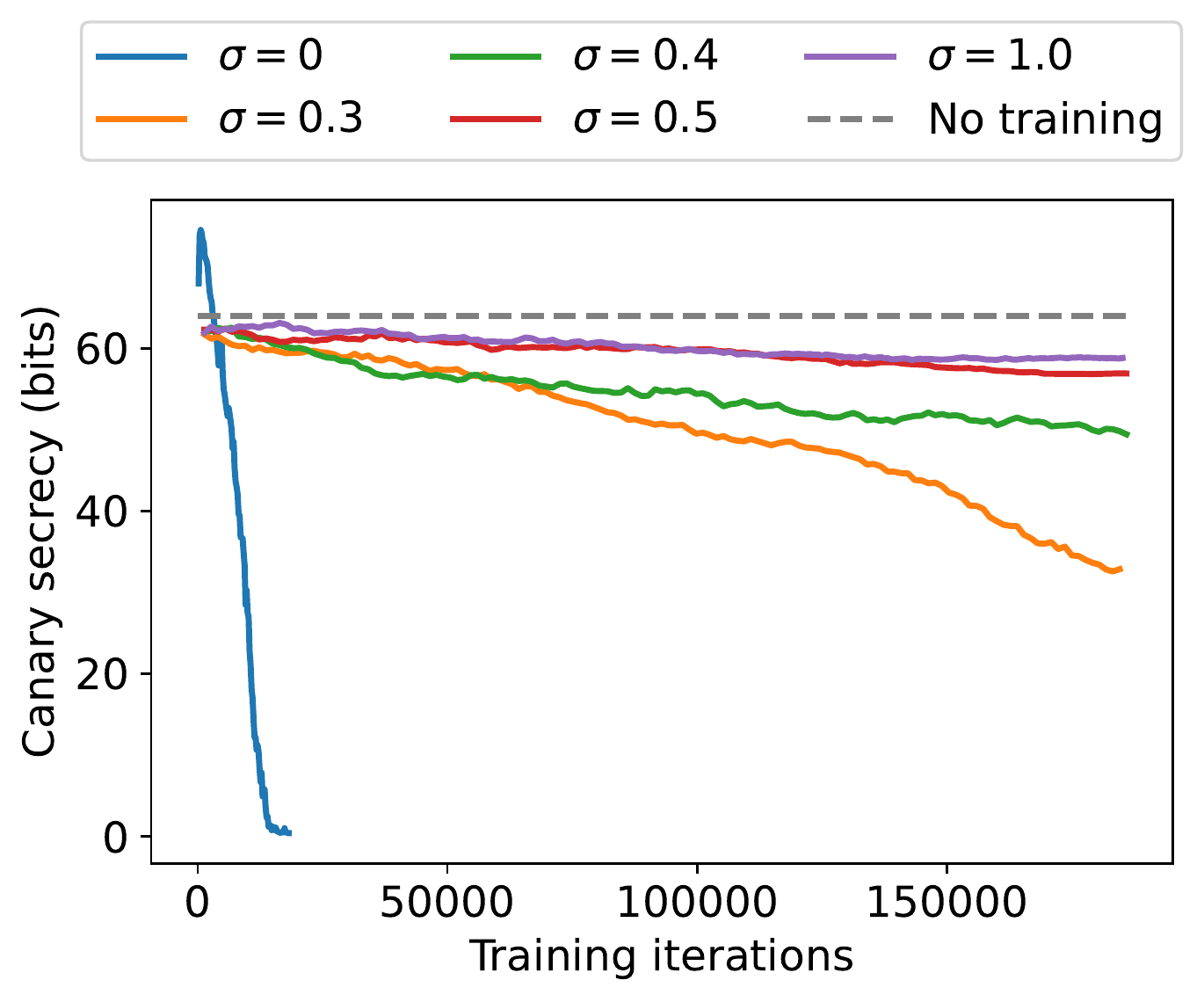}
    \vspace*{-5mm}
    \caption{We empirically measure the leakage of our canary (a fixed 16-digits random number) when training our GPT-2 model on Wikitext with DP. If the model predicts the canary with a given probability $p_1$, we compute the secrecy $b = \log_2(1/p_1)$ ($y$-axis, in bits). Higher DP noise $\sigma$ results in less memorization; the dashed curve is the secrecy of the canary without any training $b = \log_2(1/10^{16})$ (the prior $p_0$).
    With $\sigma=0$, we do not add any DP noise or gradient clipping, and therefore the model is heavily memorizing the canary, and predicts it at the end of the training with a probability $p=67\%$ (hence we stopped the training early).} 
\end{figure}

\begin{table*}[]
    \centering
    \begin{tabular}{p{0.7\linewidth}  p{0.1\linewidth} p{0.1\linewidth}}
    \toprule
    \centering
    Sample      &  Probability & Calibration \\
    \midrule
    \emph{Advertisement}  &  $10^{-2.5}$               & $+0.10$\\
    \emph{Content created by The Daily Caller News Foundation is available without charge [...]} & $10^{-3.5}$               & $-0.18$\\
    \emph{Subject: Games Day News and Rumours: Non Dark Eldar}  &  $10^{-3.6}$ & $-9.81$\\
    \emph{Antarctic expeditioners honour traditional solstice swim}  &  $10^{-10}$ & $-11.3$ \\
    \emph{Homepage image courtesy of the Canadian Space Agency}  &  $10^{-21}$ & $+1.29$ \\
    \bottomrule
    \end{tabular}
    \caption{
    \label{tab:samples}
    We train a vanilla target model without DP on Wikitext-103 and display selected sentences from the training set along with (1) the probability that the target model generates the sample and (2) the loss of the target model for this sample calibrated with the loss of a model that was not trained on this data. A low calibrated loss denotes a sample that is likely to be memorized by the target model and a high probability denotes a sample that is more likely to be generated by the target model. 
    Hence, samples associated with a low calibrated loss and with a high generation probability are considered at risk.
    Note that $-11.3$ is quite low, as most generated samples will have a calibrated loss higher than $-7.5$ as can be seen in Figure~\ref{fig:proba_diff_loss}.
    }
\end{table*}

\paragraph{Results.}
Figure~\ref{fig:sigma_leak} shows the empirical log probability of the secret as a function of the number of steps plotted against the leakage guarantee computed according to Equation~\eqref{eq:min_bound}.
We can see that without privacy ($\sigma=0$), the model almost perfectly memorizes the canary with probability $p_1 = 67\%$, as expected.
With DP, decreasing levels of the noise multiplier $\sigma$ lead to increasing levels of leakage.
Overall, even with our less private version ($\sigma=0.3$), our strategy is not able to fully reconstruct the secret.

\paragraph{Exposure.}
\citet{carlini2019secret} estimate the performance of their attack using the exposure metric, which is an approximation of the optimal number of guesses an adversary would need to correctly estimate a secret.
In practice, we can only upper-bound the exposure because we have no guarantee that any attack is optimal.
The canary secrecy showed in Figure~\ref{fig:sigma_leak} corresponds to exposure, but for the more recent reconstruction attacks of \citet{carlini2021extracting}.

\subsection{Limits of Reconstruction Attacks}

Canary attacks~\cite{carlini2019secret} are useful because the secret is chosen randomly and is thus independent of the rest of the dataset. 
In contrast, with practical reconstruction attacks it is difficult to estimate the randomness of the secret. 
One problem is that there can be multiple identical sequences in the dataset.
\citet{carlini2021extracting} correct for this by measuring $k$-eidetic memorization, \ie, looking at sentences that appear $k$ times or less in the dataset. 
Specifically, they deem sentences to be duplicates if their sets of trigrams overlap by a certain amount.
However, this does not account for knowledge shared across secrets.
For example, if there is a secret in the form ``[animal] drinks [beverage]", animal and beverage can vary in the dataset, so the language model will learn about this general structure during training: seeing ``dog drinks water" can thus make ``cat drinks milk" more likely, even if the second latter does not appear in the dataset.

\section{Conclusion}

This work shows that Rényi Differential Privacy with DP-SGD provides meaningful guarantees against reconstruction attacks.
We also provide a way to analyze the vulnerability of training samples to the reconstruction attack of \citet{carlini2021extracting}.
Overall, the combination of our improved guarantees with the private fine-tuning of language models showcased by \citet{li2021large} allow us to train language models with a perplexity competitive with the state of the art and meaningful privacy guarantees against training data extraction.
Our proposal shows that reconstruction attacks can be defeated with lower values of the noise $\sigma$, or equivalently higher signal-to-noise ratios, which should translate to better accuracies. 
Finally, our work sheds light on the ``higher information" end of the spectrum, which is of interest for two different reasons.
First, reconstruction attacks are more credible because they only require access to a trained model and not to candidate samples.
Second, as shown, reconstruction attacks can be defeated with levels of noise that would otherwise be vacuous for defending against membership inference attacks.
We hope that increased consideration for this threat model will help adoption of differential privacy as a standard in machine learning.

\clearpage
\section*{Acknowledgements}

The authors thank Nicholas Carlini and Graham Cormode for initial feedback on this paper. 

\bibliography{main}
\bibliographystyle{icml2021}

\clearpage
\appendix
\section{Appendix}

\subsection{Probability preservation}
Here we show the steps to get from Eq.~\ref{eq:logratio} to Eq.~\ref{eq:leakage_bound}:
\begin{align*}
\exp(- d_\alpha) p_0^{\alpha / (\alpha - 1)} &\leq p_1 \leq \left( \exp(d_\alpha) p_0 \right)^{(\alpha -1) / \alpha} \\
- d_\alpha + \alpha / (\alpha - 1) \log(p_0) &\leq \log(p_1) \leq ((\alpha -1) / \alpha) (d_\alpha + \log p_0) \\
-d_\alpha +1 / (\alpha - 1) \log(p_0) &\leq \log(p_1 / p_0)  \leq (\alpha -1) d_\alpha/ \alpha - 1/\alpha \log p_0.
\end{align*}

\subsection{Leakage function}
\label{app:proof_leakage}
In this part, we prove that the leakage function $L_2$ is non-decreasing and concave.
We also show that $\xi: b \mapsto L_2(b) - b$ is non-increasing.
We start with two technical lemmas.

\begin{lemma}
\label{lemma:nondecreasing}
Given a family of non-decreasing functions $f_t, t \in \calT$ , the function $f(x) \triangleq \inf_t f_t(x)$ is non-decreasing.
\end{lemma}
\begin{proof}
For $x<y$ and some $t \in \calT$, we have
\begin{align*}
    f(x) = \inf_{t'} f_{t'}(x) &\leq f_t(x) \\
    &\leq f_t(y)
\end{align*}
Since $f(x) \leq f_t(y)$ for all $t$, we have $f(x) \leq \inf_t f_t(y) = f(y)$.
\end{proof}

\begin{lemma}
\label{lemma:concave}
Given a family of concave functions $f_t, t \in \calT$ , the function $f(x) \triangleq \inf_t f_t(x)$ is concave.
\end{lemma}
\begin{proof}
For a function $g$, we define its hypograph $\calH(g) \triangleq \{ (x, y) ~|~ y \leq g(x) \}$.
A function $g$ is concave iff its hypograph $\calH(g)$ is a convex set.
Each function $f_t$ being concave, its hypograph $\calH(f_t)$ is a convex set.
The hypograph $\calH(f) = \cap_{t \in \calT} \calH(f_t)$ is an intersection of convex sets and is thus itself a convex set.
Thus $f$ is concave.
\end{proof}

Now let us apply these lemmas to the leakage function.
We have $L_2(b) = \min_\alpha f_\alpha(b) $ with each $f_\alpha(b) = d_\alpha \frac{\alpha -1}{\log(2) \alpha} + \frac{b }{\alpha}$.
Each function $f_\alpha$ is linear, and hence concave, and non-decreasing because $1/\alpha > 0$.
We can thus apply Lemma~\ref{lemma:concave} and Lemma~\ref{lemma:nondecreasing} to conclude that the leakage function $L_2$ is concave and non-decreasing.

The function $\xi: b \mapsto L_2(b) - b$ is a sum of a concave and a linear function and is thus concave. 
Thus, its derivative $\frac{\partial \xi}{\partial b}$ is non-increasing.
Given that $\frac{\partial \xi}{\partial b}(0) = \frac{\partial L_2}{\partial b}(0) - 1 = 0$, $\frac{\partial \xi}{\partial b} \leq 0$ and thus $\xi$ is non-increasing.

\end{document}